\documentclass[journal,twoside,web]{ieeecolor}
\usepackage[nocompress]{cite}
\usepackage{xcolor}
\definecolor{darkblue}{RGB}{88, 105, 187}
\definecolor{lobored}{rgb}{0.6,0,0}

\makeatletter
\let\NAT@parse\undefined
\makeatother
\usepackage[
    colorlinks=true,
    linktoc=all,
    linkcolor=darkblue,
    citecolor=lobored,
    urlcolor=lobored
]{hyperref}
\usepackage{orcidlink}
\usepackage{algorithm,algorithmic}
\usepackage{amsmath,amssymb,amsfonts,bm}
\DeclareMathOperator{\Proj}{Proj}

\usepackage{textcomp}
\usepackage{graphicx}
\usepackage{epsfig} 
\usepackage{datetime2}
\usepackage{generic}
\usepackage{lcsys}
\newtheorem{theorem}{Theorem}
\newtheorem{lemma}{Lemma}
\newtheorem{definition}{Definition}
\newtheorem{assumption}{Assumption}
\newtheorem{remark}{Remark}
\newtheorem{corollary}{Corollary}
\makeatletter
\def\@begintheorem#1#2{\@IEEEtmpitemindent\itemindent\topsep 0pt\rmfamily\trivlist%
    \item[\hskip\labelsep{\itshape #1\ #2:}]\itemindent\@IEEEtmpitemindent}
\def\@opargbegintheorem#1#2#3{\@IEEEtmpitemindent\itemindent\topsep 0pt\rmfamily\trivlist%
    \item[\hskip\labelsep{\itshape #1\ #2\ (#3):}]\itemindent\@IEEEtmpitemindent}

\makeatother
\begin{document}

\def\BibTeX{{\rm B\kern-.05em{\sc i\kern-.025em b}\kern-.08em
    T\kern-.1667em\lower.7ex\hbox{E}\kern-.125emX}}
\markboth{\journalname, VOL. XX, NO. XX, XXXX 2017}
{Author \MakeLowercase{\textit{et al.}}: Preparation of Papers for IEEE Control Systems Letters (August 2022)}

\title{
LoRA-RC: Reservoir Computing with \\ Low-Rank Adaptation
}

\author{
Wenbin Wan \orcidlink{0000-0002-4920-2215}, Member, IEEE
\thanks{
This work was supported in part by the University of New Mexico under SOE faculty startup funding and NM SPARK Scholars Program.}
\thanks{Wenbin Wan is with the Department of Mechanical Engineering, University of New Mexico, Albuquerque, NM 87131, USA. \texttt{wwan@unm.edu}}
}
\maketitle
\thispagestyle{empty}
\begin{abstract}
Reservoir computing (RC) trains only a linear readout over a fixed recurrent layer, making it fast and data-efficient for online prediction. However, a static reservoir degrades under system drift, readout-only adaptation is then insufficient, and unconstrained reservoir adaptation can destroy the echo-state and incremental stability properties that make RC reliable. This paper proposes LoRA-RC, which adapts the recurrent matrix through a low-rank correction driven by streaming prediction errors. The base reservoir and adaptation bases are fixed offline; a small core matrix is adapted online, projected onto a spectral-norm ball, and low-pass filtered at each step. The projection guarantees that every applied recurrent matrix remains within a certified contraction set, and an incremental input-to-state stability bound is established for the reservoir along each online adaptation path, with path-independent rate and gain. On a Lorenz system with an abrupt parameter drift, LoRA-RC cuts post-drift prediction error by 56\% versus a fixed RC and 51\% versus readout-only adaptation; ablations over 20 seeds show that removing the projection inflates this error by more than a factor of 40.
\end{abstract}


\section{Introduction}

\IEEEPARstart{L}{earning-enabled} predictors are increasingly used in cyber-physical systems (CPS) when first-principles models are incomplete or unavailable. For safety-critical operation, however, accurate nominal prediction is not enough. A deployable predictor must have a well-behaved internal state and must respond in a controlled way to sensor perturbations. These requirements connect learning-enabled prediction to classical control notions such as input-to-state stability, incremental stability, robustness, and certifiable operation under uncertainty~\cite{sontag1989smooth,jiang2001discreteISS,angeli2002incremental,mayne2000mpc}. Online adaptation is often necessary in deployed CPS. A predictor trained offline may work well on nominal data, but payload, contact conditions, plant parameters, and sensing channels can drift, making a static learned predictor unreliable even after successful validation. Fully training a recurrent neural network online could increase flexibility, but it is difficult to certify during operation. Updating many recurrent parameters can introduce unstable hidden dynamics, uncontrolled sensitivity to corrupted measurements, and failure modes that are hard to diagnose.

Reservoir computing (RC) offers a useful middle ground. In an echo state network (ESN), a high-dimensional recurrent reservoir encodes temporal information, while only a simple readout is trained \cite{jaeger2001echo,lukosevicius2009reservoir}. This separation makes training fast and data-efficient. It also gives the model enough structure for control-oriented stability analysis. Online readout adaptation has been used for nonlinear system identification \cite{jaeger2002adaptive,yang2019online}, and reservoir design is known to affect memory, nonlinearity, and prediction performance \cite{gallicchio2011architectural}. These properties make RC attractive for real-time prediction in CPS.

This simplicity also creates a limitation. A fixed reservoir can be too rigid when the system dynamics change. If the operating regime leaves the feature space represented by the nominal reservoir, readout-only adaptation may not be enough because the temporal memory itself remains unchanged. Readout adaptation can only reweight the features generated by the existing reservoir, whereas recurrent adaptation changes the reservoir memory map that generates those features.
\begin{figure}[t]
    \centering
    \includegraphics[width=1\linewidth]{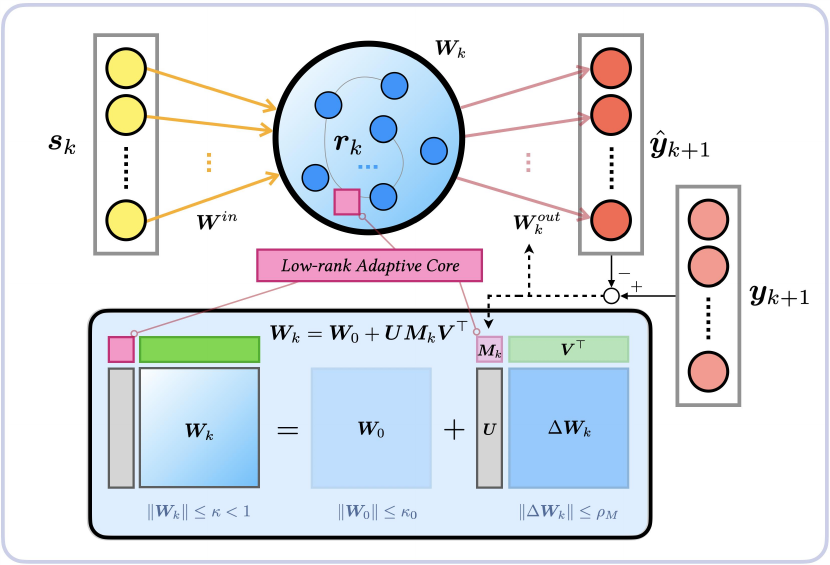}
    \caption{LoRA-RC architecture.
    \textit{Top}: the reservoir receives input $\bm{s}_k$ through $\bm{W}^{in}$, evolves hidden state $\bm{r}_k$ under recurrent matrix $\bm{W}_k$, and produces one-step-ahead prediction $\hat{\bm{y}}_{k+1}$ through time-varying readout $\bm{W}^{out}_k$; the residual $\bm{e}_{k+1}=\bm{y}_{k+1}-\hat{\bm{y}}_{k+1}$ drives the online updates.
    \textit{Bottom}: the recurrent matrix is decomposed as $\bm{W}_k = \bm{W}_0 + \bm{U}\bm{M}_k\bm{V}^\top$, where the fixed bases $\bm{U}$ and $\bm{V}$ are chosen offline and the small core $\bm{M}_k\in\mathbb{R}^{r\times r}$ is updated online.}
    \label{fig:arch}
    \vspace{-3mm}
\end{figure}
A natural solution is to adapt the recurrent matrix. Recent adaptive RC makes the recurrent connectivity responsive through excitatory--inhibitory control \cite{srinivasan2025boosting}, robot--reservoir timescale alignment \cite{ye2025reservoir}, and chaos-controlled reservoirs in living neurons \cite{kim2026living}, but attaches no stability certificate to the update.
The recurrent matrix determines how past information is stored, forgotten, and amplified. Unconstrained recurrent updates can destroy the echo-state property or make the predictor amplify corrupted inputs. The usual spectral-radius heuristic is not sufficient, since ESNs can fail to have the echo-state property even when the spectral radius is below one \cite{yildiz2012revisiting}. Input-dependent views of the echo-state property also show that stability depends on both the reservoir and the driving signal \cite{manjunath2013echo}. These results point to the same lesson: online reservoir adaptation needs a stability-aware update law instead of unconstrained recurrent tuning \cite{grigoryeva2018universal,dong2022asymptotic}.

Control-oriented recurrent-model research makes this gap more precise. Fixed RC models have been certified through incremental input-to-state stability ($\delta$ISS), and norm-type conditions have been used in observer and model predictive control designs \cite{bugliari2019mpc}, an idea extended by more recent echo-state predictive control \cite{williams2026reservoir,inoue2025reservoir}. RC has likewise been pushed toward deployment through online recalibration and distribution-free uncertainty quantification \cite{chen2022calibrated,neglia2026rescp} and model-free reservoir observers with online weight updates \cite{lu2017reservoir,tavakoli2026timedelay}, which adapt the readout or quantify uncertainty rather than certify an adapting recurrent matrix. ISS and $\delta$ISS certificates have since been established for gated recurrent units and broader offline-constrained RNN classes \cite{bonassi2021gru,damico2024incremental,bonassi2024nmpc}; regional Lyapunov conditions have extended these results to nonlinear model predictive control \cite{labella2025regional}. These results place recurrent learning models within the language of control theory. Yet they mainly certify fixed, offline-trained, or offline-constrained recurrent models. They do not provide an online reservoir adaptation rule that changes the recurrent memory while keeping a stability certificate invariant after every update.
This paper addresses this gap by developing LoRA-RC, a stability-preserving low-rank adaptive reservoir architecture. A natural starting point is low-rank adaptation (LoRA), popularized as a parameter-efficient fine-tuning technique for large language models~\cite{hu2022lora}, which freezes a large base matrix and modifies it through a small trainable correction of the form
$\bm{W}_{\mathrm{new}}=\bm{W}_0+\bm{A}\bm{B}^{\top}$, where $\bm{A}\bm{B}^{\top}$ has rank much smaller than the dimension of $\bm{W}_0$. Recent LoRA variants further improve efficiency by adapting rank allocation, preserving prior knowledge, or routing low-rank updates through structured mechanisms~\cite{zhang2023adalora,zou2025flylora,luo2026keeplora}. These methods show the value of low-dimensional parameter adaptation, but they mainly target static or offline-trained models.

A reservoir predictor raises a different issue. Its recurrent matrix is not only a parameter map. It is the memory map of a dynamical system. If a LoRA-style correction is applied directly during streaming deployment, two difficulties arise. First, updating both low-rank factors online makes the recurrent dynamics harder to certify, because the adaptation directions themselves change during operation. Second, standard low-rank adaptation does not provide a mechanism to keep the updated recurrent matrix inside a certified contraction set, for example $\|\bm{W}_{\mathrm{new}}\|<1$, after every online update. As a result, an unconstrained low-rank correction can improve local prediction loss while degrading the stability property required for trustworthy reservoir prediction.

LoRA-RC addresses this issue by separating offline adaptation directions from the online adaptive variable. The recurrent matrix is decomposed as $\bm{W}_k=\bm{W}_0+\bm{U}\bm{M}_k\bm{V}^{\top}$, where $\bm{W}_0$ is a certified base reservoir, $\bm{U}$ and $\bm{V}$ are fixed bases chosen offline, and the small core $\bm{M}_k$ is updated online. After each gradient step, the core is projected onto a spectral-norm set, and a first-order filter generates the applied correction. This construction preserves the tractable structure of reservoir computing while allowing the recurrent memory to adapt within a certified low-rank subspace. It changes low-rank adaptation from an unconstrained parameter update into a certified recurrent-memory update. In particular, the projection keeps each applied recurrent matrix inside a contraction set, while the filter controls the step-to-step variation of the online reservoir. As a result, the time-varying reservoir can adapt its memory under streaming prediction errors while preserving an incremental input-to-state stability certificate that holds uniformly over each online adaptation path, with rate and gain that do not depend on which path the update law takes. The architecture is shown in Fig.~\ref{fig:arch} and described in Section~\ref{sec:arch}. This work makes three contributions.
\begin{itemize}
    \item A low-rank adaptive reservoir computing architecture is introduced whose recurrent matrix is decomposed into a certified base reservoir and a projected and filtered online correction. This structure separates offline selection of adaptation directions from online updating of a small core matrix.
    \item An online update law for the readout and the low-rank core matrix is derived using projected gradient descent and first-order filtering. The update changes the reservoir memory while keeping the applied recurrent matrix inside the certified contraction set.
    \item The resulting time-varying reservoir is proved to be uniformly $\delta$ISS along each online adaptation path, with path-independent contraction rate and input gain. A direct bound on the step-to-step variation of the applied recurrent matrix is also derived as a function of the filter coefficient, and the contraction rate is shown to carry through to the prediction layer.
\end{itemize}

\section{Preliminaries} \label{sec:notation}

\textit{Notation.}
The subscript $k$ denotes the discrete time index. Vectors are written in bold lowercase, e.g., $\bm{r}_k\in\mathbb{R}^n$; matrices are written in bold uppercase, e.g., $\bm{W}\in\mathbb{R}^{n\times n}$. The notation $\bm{A}^{\top}$ denotes the transpose of a matrix $\bm{A}$. The norm $\|\cdot\|$ denotes the Euclidean norm for vectors and the induced matrix 2-norm (spectral norm) for matrices; $\|\cdot\|_F$ denotes the Frobenius norm. For a sequence of vectors $\{\bm{s}_k\}_{k\geq 0}$, $\|\bm{s}\|_\infty=\sup_{k\geq 0}\|\bm{s}_k\|$. The symbol $\odot$ denotes the elementwise (Hadamard) product. The activation function $\sigma(\cdot)$ is applied componentwise; $\sigma'(\cdot)$ denotes its componentwise derivative. $\bm{I}$ denotes the identity matrix of appropriate size. For a set $\mathcal{M}$ and a matrix $\bm{Z}$, $\Proj_{\mathcal{M}}(\bm{Z})$ denotes the Frobenius-distance projection of $\bm{Z}$ onto $\mathcal{M}$.

\textit{Reservoir Computing (RC).}
Let $\bm{r}_k\in\mathbb R^n$ denote the reservoir state. A standard RC uses the leaky reservoir model
\begin{equation}
    \bm{r}_{k+1}=(1-\alpha)\bm{r}_k+\alpha \sigma(\bm{W}_k \bm{r}_k+\bm{W}^{in} \bm{s}_k),
    \label{eq:reservoir}
\end{equation}
where $\alpha\in(0,1]$ is the leaking rate, $\bm{W}_k\in\mathbb{R}^{n\times n}$ is the recurrent weight matrix, $\bm{W}^{in}\in\mathbb R^{n\times p}$ is the fixed input weight matrix, and $\sigma$ is a componentwise nonlinear activation. The following assumption on $\sigma$ is standard and is satisfied by $\tanh$ and unit-Lipschitz saturating activations.
\begin{assumption}[Lipschitz activation]
\label{ass:lipschitz}
The activation function $\sigma$ is componentwise Lipschitz with constant $L_\sigma\leq 1$.
\end{assumption}

The prediction is
$
    \hat{\bm{y}}_{k+1}=\bm{W}^{out}_k\bm{r}_{k+1},
$
where $\bm{W}^{out}_k\in\mathbb R^{q\times n}$ is the readout matrix.
A standard RC fixes $\bm{W}_k=\bm{W}_0$ and trains only the readout offline by solving the ridge regression problem
$
    \bm{W}^{out}_0 = \arg\min_{\bm{W}}\sum_k\|\bm{y}_{k+1}-\bm{W}\bm{r}_{k+1}\|^2+\lambda\|\bm{W}\|_F^2,
$
whose closed-form solution is
$\bm{W}^{out}_0=\bm{Y}\bm{R}^{\top}(\bm{R}\bm{R}^{\top}+\lambda \bm{I})^{-1},$
where the columns of $\bm{R}$ are $\bm{r}_{k+1}$, the columns of $\bm{Y}$ are $\bm{y}_{k+1}$, and $\lambda>0$ is the ridge regularization coefficient.

\textit{Problem Formulation.}
Consider a CPS that generates a measured signal sequence $\bm{s}_k\in\mathbb R^p$ and an output sequence $\bm{y}_k\in\mathbb R^q$.
The signal $\bm{s}_k$ may include measured inputs, past outputs, or exogenous variables.
The goal is one-step-ahead prediction,
$\hat{\bm{y}}_{k+1}=f_k(\bm{s}_0,\ldots,\bm{s}_k),$
where the function $f_k$ is allowed to update online. An adaptive predictor with two properties is sought. First, it should improve prediction when the nominal dynamics drift. Second, it should preserve internal stability during online adaptation.

\section{Low-Rank Adaptive Reservoir Computing (LoRA-RC) Architecture} \label{sec:arch}

The LoRA-RC architecture augments a stable base reservoir with a low-rank recurrent correction that is adapted online, while constraining the correction to a set in which the contraction property of the reservoir is preserved at every step. Specifically, LoRA-RC constrains the recurrent matrix to the form
$\bm{W}_k=\bm{W}_0+\bm{U}\bm{M}_k \bm{V}^{\top}$,
where $\bm{W}_0\in\mathbb R^{n\times n}$ is a certified base reservoir, $\bm{U},\bm{V}\in\mathbb R^{n\times r}$ are fixed orthonormal adaptation bases selected offline, $\bm{M}_k\in\mathbb R^{r\times r}$ is the applied core updated online, and $r\ll n$. To make the contraction property invariant under online updates, the applied core is restricted to a norm-bounded set, defined as follows.

\begin{definition}
\label{def:admissible}
For constants $0<\kappa_0<\kappa<1$ and $\rho_M=\kappa-\kappa_0$, the \emph{admissible set} is
$
    \mathcal M=\{\bm{M}\in\mathbb R^{r\times r}:\|\bm{M}\|\leq \rho_M\}
$.
A core $\bm{M}$ is called \emph{admissible} whenever $\bm{M}\in\mathcal M$.
\end{definition}

With the admissible set in place, the offline and online roles of the decomposition can be stated.
The base reservoir $\bm{W}_0$, the adaptation bases $\bm{U}$ and $\bm{V}$, and the admissible set $\mathcal{M}$ are all fixed offline. Keeping $\bm{U}$ and $\bm{V}$ fixed is what makes the certificate invariant by construction: if those two were also updated online, the method would reduce to a partially trained recurrent network with no closed-form stability guarantee. During online operation, only the small core matrix is updated, while the readout is updated separately. The online core update introduces an auxiliary variable $\hat{\bm{M}}_k\in\mathbb{R}^{r\times r}$, called the \emph{fast core}, which is the raw gradient iterate before filtering. After each gradient step, the tentative fast-core iterate is projected onto the admissible set $\mathcal{M}$, which, under the orthonormal-basis condition, keeps every applied matrix $\bm{W}_k$ inside a certified contraction set at every step. A first-order low-pass filter then maps $\hat{\bm{M}}_k$ to the applied core $\bm{M}_k$ used in $\bm{W}_k$, decoupling the learning rate from the rate at which the reservoir memory changes.

The structural conditions on the offline components and the initialization of the online cores are collected in the following three assumptions, which are used throughout the remainder of the paper.

\begin{assumption}
\label{ass:base}
The base reservoir is stable with $\|\bm{W}_0\|\leq \kappa_0$.
\end{assumption}

\begin{assumption}
\label{ass:init}
The initial cores satisfy $\hat{\bm{M}}_0,\bm{M}_0\in\mathcal M$.
\end{assumption}

\begin{assumption}
\label{ass:bases}
The input matrix $\bm{W}^{in}$ is bounded, and the bases $\bm{U},\bm{V}\in\mathbb{R}^{n\times r}$ are orthonormal, $\bm{U}^{\top}\bm{U}=\bm{V}^{\top}\bm{V}=\bm{I}$.
\end{assumption}

\begin{algorithm}[t]
\footnotesize
\caption{LoRA-RC Offline Design}
\label{alg:offline}
\begin{algorithmic}[1]
\REQUIRE Training data $\mathcal{D}_{\rm tr}=\{(\bm{s}_k,\bm{y}_{k+1})\}$; hyperparameters $\kappa_0,\kappa,r,\lambda$; optional drift regimes $\mathcal{D}^{(1)},\ldots,\mathcal{D}^{(N_b)}$ with regularization weight $\lambda_W$
\ENSURE $\bm{W}_0,\bm{U},\bm{V},\bm{W}^{in},\bm{W}^{out}_0,\hat{\bm{M}}_0,\bm{M}_0$
\STATE $\rho_M \leftarrow \kappa-\kappa_0$, $\hat{\bm{M}}_0\leftarrow\bm{0}$, $\bm{M}_0\leftarrow\bm{0}$
\STATE Draw a sparse random matrix $\bar{\bm{W}}$ \STATE $\bm{W}_0 \leftarrow (\kappa_0/\|\bar{\bm{W}}\|)\bar{\bm{W}}$
\STATE Draw a random input matrix $\bm{W}^{in}$
\STATE Run reservoir \eqref{eq:reservoir} on $\mathcal{D}_{\rm tr}$
\STATE Train $\bm{W}^{out}_0$ by ridge regression
\IF{drift regimes are available}
    \FOR{$j=1,\ldots,N_b$}
        \STATE Solve the constrained batch correction problem for $\Delta\bm{W}^{(j)}$ by projected gradient, maintaining $\|\bm{W}_0+\Delta\bm{W}^{(j)}\|\leq\kappa$
    \ENDFOR
    \STATE $\overline{\Delta\bm{W}} \leftarrow N_b^{-1}\sum_j \Delta\bm{W}^{(j)}$
    \STATE Set $\bm{U},\bm{V}\in\mathbb{R}^{n\times r}$ to the leading $r$ left/right singular vectors of $\overline{\Delta\bm{W}}$
\ELSE
    \STATE Draw random orthonormal $\bm{U},\bm{V}\in\mathbb{R}^{n\times r}$
\ENDIF
\end{algorithmic}
\end{algorithm}

\begin{remark}
\label{rem:assumptions}
Assumptions~\ref{ass:base}--\ref{ass:bases} are enforced by the offline design rather than imposed on the data: spectral normalization sets $\|\bm{W}_0\|=\kappa_0$, the zero initialization $\hat{\bm{M}}_0=\bm{M}_0=\bm{0}\in\mathcal{M}$ ensures initial admissibility, and $\bm{U},\bm{V}$ are taken as singular vectors (or random orthonormal matrices) and held fixed together with $\bm{W}^{in}$.
\end{remark}

\subsection{Offline Initialization}

The offline stage uses nominal training data $\mathcal D_{\rm tr}=\{(\bm{s}_k,\bm{y}_{k+1})\}_{k=0}^{T-1}$.
First, draw a sparse random matrix $\bar{\bm{W}}$ and a random input matrix $\bm{W}^{in}$, then set
$\bm{W}_0=\frac{\kappa_0}{\|\bar{\bm{W}}\|}\bar{\bm{W}}$.
The remaining offline step is to choose the adaptation bases $\bm{U}$ and $\bm{V}$, which define the subspace in which the reservoir can be corrected online.
Writing $\bm{U}=[\bm{u}_1,\ldots,\bm{u}_r]$ and $\bm{V}=[\bm{v}_1,\ldots,\bm{v}_r]$, the correction expands as $\bm{U}\bm{M}_k\bm{V}^{\top}=\sum_{i,j}(\bm{M}_k)_{ij}\,\bm{u}_i\bm{v}_j^{\top}$: $\bm{V}^{\top}\bm{r}_k$ reads the state along the input-side directions, $\bm{M}_k$ recombines them, and $\bm{U}$ writes the result along the output-side directions. The column space of $\bm{U}\bm{M}_k\bm{V}^{\top}$ therefore lies in $\mathrm{span}(\bm{U})$ and its row space in $\mathrm{span}(\bm{V})$, regardless of $\bm{M}_k$.
By default, $\bm{U},\bm{V}\in\mathbb{R}^{n\times r}$ are random orthonormal matrices, where $r$ is the retained adaptation rank; this requires no prior knowledge of the drift.
If representative drift regimes $\mathcal D^{(1)},\ldots,\mathcal D^{(N_b)}$ are available offline, a data-informed alternative replaces random basis selection.
For each regime $j$, compute a stable batch reservoir correction
$\Delta\bm{W}^{(j)}=\arg\min_{\Delta\bm{W}}\;
    \sum_k \bigl\|\bm{y}_{k+1}^{(j)}
         -\hat{\bm{y}}_{k+1}^{(j)}  \bigr\|^2
    +\lambda_W\|\Delta\bm{W}\|_F^2$
subject to $\|\bm{W}_0+\Delta\bm{W}\|\leq\kappa$, where $\lambda_W$ is the regularization weight and the predictions $\hat{\bm{y}}^{(j)}_{k+1}$ depend on $\bm{W}_0+\Delta\bm{W}$ that is used as the recurrent matrix in \eqref{eq:reservoir}, generating states $\bm{r}^{(j)}_{k+1}$ over regime $j$, which are mapped to outputs by a ridge readout refit on that regime at each gradient step.
The constraint keeps the corrected reservoir inside the certified contraction set, so the directions are fitted from a reservoir with the echo-state property.

The problem is solved approximately by projected gradient steps.
The bases $\bm{U},\bm{V}$ are then taken as the leading $r$ left and right singular vectors of the average correction $\overline{\Delta\bm{W}}=N_b^{-1}\sum_{j=1}^{N_b}\Delta\bm{W}^{(j)}$, concentrating the adaptation budget on the directions that were most useful across the observed drift scenarios.
The fast and applied cores are initialized as $\hat{\bm{M}}_0=\bm{M}_0=\bm{0}$.
During deployment $\bm{U},\bm{V}$ remain fixed and only $\bm{W}^{out}_k$, $\hat{\bm{M}}_k$, and $\bm{M}_k$ are updated; Algorithm~\ref{alg:offline} summarizes the offline design.

\subsection{Online Prediction and Update}

\begin{algorithm}[t]
\footnotesize
\caption{LoRA-RC Online Update}
\label{alg:online}
\begin{algorithmic}[1]
\REQUIRE Fixed $\bm{W}_0,\bm{U},\bm{V},\bm{W}^{in},\rho_M,\bar R,\alpha,\eta_R,\eta_M,\beta,\epsilon,\lambda_M$; state $\bm{W}^{out}_k,\hat{\bm{M}}_k,\bm{M}_k,\bm{r}_k$; input $\bm{s}_k$
\ENSURE Updated $\bm{W}^{out}_{k+1},\hat{\bm{M}}_{k+1},\bm{M}_{k+1},\bm{r}_{k+1},\bm{W}_{k+1}$
\STATE $\bm{W}_k \leftarrow \bm{W}_0+\bm{U}\bm{M}_k\bm{V}^{\top}$
\STATE $\bm{a}_k \leftarrow \bm{W}_k\bm{r}_k+\bm{W}^{in}\bm{s}_k$
\STATE $\bm{r}_{k+1} \leftarrow (1-\alpha)\bm{r}_k+\alpha\sigma(\bm{a}_k)$
\STATE $\hat{\bm{y}}_{k+1} \leftarrow \bm{W}^{out}_k\bm{r}_{k+1}$
\STATE Observe $\bm{y}_{k+1}$; \quad $\bm{e}_{k+1} \leftarrow \bm{y}_{k+1}-\hat{\bm{y}}_{k+1}$
\STATE $\bm{W}^{out}_{k+1} \leftarrow \Proj_{\bar R}\!\bigl(\bm{W}^{out}_k + \eta_R\,\bm{e}_{k+1}\bm{r}_{k+1}^{\top}/(\epsilon+\|\bm{r}_{k+1}\|^2)\bigr)$
\STATE $\bm{G}^M_k \leftarrow \bm{U}^{\top}(\bm{\delta}^a_k\bm{r}_k^{\top})\bm{V}+\lambda_M\hat{\bm{M}}_k$
\STATE $\hat{\bm{M}}_{k+1} \leftarrow \Proj_{\mathcal{M}}\!\left(\hat{\bm{M}}_k - \eta_M\bm{G}^M_k\right)$
\STATE $\bm{M}_{k+1} \leftarrow (1-\beta)\bm{M}_k+\beta\hat{\bm{M}}_{k+1}$
\STATE $\bm{W}_{k+1} \leftarrow \bm{W}_0+\bm{U}\bm{M}_{k+1}\bm{V}^{\top}$
\end{algorithmic}
\end{algorithm}

Each online step has three ordered operations: predict with the current applied core $\bm{M}_k$, observe the true output to form the residual, and update the readout and small core matrix.
After predicting $\hat{\bm{y}}_{k+1}$, the true output $\bm{y}_{k+1}$ becomes available.
Define the residual
$\bm{e}_{k+1}=\bm{y}_{k+1}-\hat{\bm{y}}_{k+1}$.
The readout update is a normalized stochastic-gradient step, projected onto a spectral-norm ball of radius $\bar R>0$,
\begin{equation}
    \bm{W}^{out}_{k+1}=\Proj_{\bar R} \ (\bm{W}^{out}_k+\eta_R\frac{\bm{e}_{k+1}\bm{r}_{k+1}^{\top}}{\epsilon+\|\bm{r}_{k+1}\|^2}),
    \label{eq:readout_update}
\end{equation}
where $\eta_R>0$ is the readout step size, $\epsilon>0$ is a small constant that prevents division by a vanishing reservoir-state norm, and $\Proj_{\bar R}(\cdot)$ is the Frobenius-distance projection introduced in Section~\ref{sec:notation}. Initializing $\bm{W}^{out}_0$ with $\|\bm{W}^{out}_0\|\leq\bar R$ then gives $\|\bm{W}^{out}_k\|\leq\bar R$ for all $k\geq 0$, since every later iterate is an output of $\Proj_{\bar R}$ and therefore lies in the ball.
The reservoir update is applied at each step alongside the readout update.
Since $\bm{U}$ and $\bm{V}$ are fixed, the gradient only needs to be computed with respect to a small $r\times r$ core rather than all $n^2$ recurrent entries.

The adaptation signal for the core is derived from the one-step residual loss $J_k=\frac{1}{2}\|\bm{y}_{k+1}-\hat{\bm{y}}_{k+1}\|^2$ by propagating the residual $\bm{e}_{k+1}$ back through the readout and the activation. Define the pre-activation of the reservoir at time $k$ as $\bm{a}_k=\bm{W}_k \bm{r}_k+\bm{W}^{in} \bm{s}_k$.
The readout $\bm{W}^{out}_k$ used below is the one that produced the prediction at time $k$, before any update at that same step.
The backpropagated local signal is
$\bm{\delta}^a_k=-\alpha\left(\sigma'(\bm{a}_k)\odot (\bm{W}^{out}_k)^{\top} \bm{e}_{k+1}\right)$,
which is the gradient of $J_k$ with respect to $\bm{a}_k$ scaled by the leak rate $\alpha$.
Projecting the outer product $\bm{\delta}^a_k\bm{r}_k^{\top}$ onto the adaptation subspace and adding a leakage term gives
the core adaptation signal,
\begin{equation}
    \bm{G}^M_k=\bm{U}^{\top}\left(\bm{\delta}^a_k \bm{r}_k^{\top}\right)\bm{V}+\lambda_M \hat{\bm{M}}_k,
    \label{eq:gradM}
\end{equation}
where $\lambda_M\geq 0$ is a leakage coefficient and $\lambda_M\hat{\bm{M}}_k$ is the corresponding leakage term that discourages drift of the fast adaptive core.
This update can be interpreted as a projected-gradient step on an instantaneous surrogate loss with respect to the applied core, followed by filtering; it does not differentiate through the filter dynamics over multiple future steps.

The fast core is then updated by projected gradient descent onto $\mathcal{M}$ with step size $\eta_M>0$,
\begin{equation}
    \hat{\bm{M}}_{k+1}=\Proj_{\mathcal M}\left(\hat{\bm{M}}_k-\eta_M \bm{G}^M_k\right).
    \label{eq:projected_update}
\end{equation}

\vspace{-2mm}

\begin{remark}
\label{rem:projection_wellposed}
Since $\rho_M>0$, the admissible set $\mathcal M$ is a closed, bounded, convex, and nonempty spectral-norm ball in $\mathbb R^{r\times r}$; these properties guarantee that the Frobenius-distance projection $\Proj_{\mathcal M}(\cdot)$ in \eqref{eq:projected_update} is well-defined and single-valued.
\end{remark}

The applied core is then generated by a first-order low-pass filter with coefficient $\beta\in(0,1]$,
\begin{equation}
    \bm{M}_{k+1}=(1-\beta)\bm{M}_k+\beta\hat{\bm{M}}_{k+1},\qquad 0<\beta\leq 1.
    \label{eq:filtered_update}
\end{equation}
Unrolling the recursion (with $\hat{\bm{M}}_0= \bm{M}_0=\bm{0}$) gives $\bm{M}_k=\beta\sum_{\ell=1}^{k}(1-\beta)^{k-\ell}\,\hat{\bm{M}}_\ell$, so the applied core at step $k$ is an exponentially weighted average of all past fast-core values, with weights on older values decaying as $(1-\beta)^{k-\ell}$. The filter acts entrywise on $\bm{M}_k$, applying the same scalar recursion to each $(i,j)$ entry.
Smaller $\beta$ gives slower recurrent adaptation with smaller step-to-step variation in $\bm{W}_k$.
For the spectral-norm ball $\mathcal M=\{\bm{M}:\|\bm{M}\|\leq\rho_M\}$, the projection has a closed form via singular-value clipping.
Let $\tilde{\bm{M}}=\hat{\bm{M}}_k-\eta_M\bm{G}^M_k$ denote the unprojected iterate inside $\Proj_{\mathcal M}(\cdot)$ in \eqref{eq:projected_update}, and let $\tilde{\bm{M}}=\bm{P}\bm{\Sigma}\bm{Q}^{\top}$ be its singular value decomposition, where $\bm{\Sigma}$ is the diagonal matrix of singular values $\bm{\Sigma}_{ii}\geq 0$.
Since the spectral norm $\|\tilde{\bm{M}}\|$ equals the largest singular value, the Frobenius-closest matrix with spectral norm at most $\rho_M$ is obtained by capping each singular value at $\rho_M$, i.e.,
$\Proj_{\mathcal M}(\tilde{\bm{M}}) = \bm{P}\bar{\bm{\Sigma}}\bm{Q}^{\top}, \bar{\bm{\Sigma}}_{ii} = \min\{\bm{\Sigma}_{ii},\rho_M\}.$
Algorithm~\ref{alg:online} summarizes the online prediction and update procedure.

\begin{table*}[t]
\caption{Lorenz drift experiments: (a) method comparison, (b) component ablation, and (c) hyperparameter sensitivity}
\label{tab:results}
\begin{center}
\vspace{-4mm}
\footnotesize
\setlength{\tabcolsep}{3.5pt}
\begin{minipage}[t]{0.30\textwidth}
\centering
(a) Method comparison (single run)\\[3pt]
\begin{tabular}{lccc}
\hline
       & \multicolumn{2}{c}{RMSE} & LoRA-RC \\
\cline{2-3}
Method & Pre & Post & reduction  \\
\hline
Fixed RC            & 0.313 & 1.372 & 56\% \\
Adaptive RC         & 0.309 & 1.211 & 51\% \\
\textbf{LoRA-RC}    & \textbf{0.235} & \textbf{0.598} & --- \\
\hline
\end{tabular}
\end{minipage}
\hfill
\begin{minipage}[t]{0.36\textwidth}
\centering
(b) Ablation, 20 seeds (mean $\pm$ std)\\[3pt]
\begin{tabular}{lcc}
\hline
Condition & Pre-drift & Post-drift \\
\hline
No projection            & 27.07$\pm$8.16 & 26.51$\pm$8.13 \\
No filter                & 3.90$\pm$0.75  & 3.53$\pm$0.68  \\
LoRA-RC ($r{=}1$)        & 0.264$\pm$0.113 & 1.112$\pm$0.183 \\
\textbf{LoRA-RC ($r{=}5$)}        & \textbf{0.237$\pm$0.055} & \textbf{0.629$\pm$0.156} \\
\hline
\end{tabular}
\end{minipage}
\hfill
\begin{minipage}[t]{0.30\textwidth}
\centering
(c) Sensitivity, 20 seeds\\[3pt]
\begin{tabular}{lcc}
\hline
Parameter & Range swept & Post-drift RMSE \\
\hline
$\rho_M$    & $0.05$--$0.35$ & $0.624$--$0.865$ \\
$\beta$     & $0.02$--$1.0$  & $0.629$--$3.532$ \\
$\eta_M$    & $0.01$--$0.16$ & $0.597$--$0.686$ \\
$\lambda_M$ & $0$--$240$     & $0.597$--$0.917$ \\
\hline
\end{tabular}
\end{minipage}
\vspace{-5mm}
\end{center}
\end{table*}
\subsection{Stability Analysis}

The analysis uses the standard incremental-stability viewpoint \cite{angeli2002incremental}.
Prior results certify a fixed or offline-trained system once, via contraction for ESNs \cite{bugliari2019mpc}, LMI conditions for broader RNN classes \cite{damico2024incremental}, or an incremental-stability hypothesis for controller design \cite{bonassi2024nmpc}.
None of these certificates is invariant under online changes of the recurrent matrix. The contribution is therefore not just an integration, but an update law that keeps every applied $\bm{W}_k$ in a norm-bounded family, extending the certificate from a fixed weight to the whole online sequence $\{\bm{W}_k\}$.
That invariance yields an explicit rate and input gain, uniform over core sequences, via a direct trajectory argument rather than a Lyapunov construction.

\begin{assumption}
\label{ass:update}
At each time step, recurrent adaptation is either frozen, so $\hat{\bm{M}}_{k+1}=\hat{\bm{M}}_k$ and $\bm{M}_{k+1}=\bm{M}_k$, or it is performed by \eqref{eq:projected_update}--\eqref{eq:filtered_update}.
\end{assumption}

\begin{lemma}
\label{lem:admissible_invariance}
Under Assumptions~\ref{ass:init}--\ref{ass:update}, $\hat{\bm{M}}_k\in\mathcal M$ and $\bm{M}_k\in\mathcal M$ for all $k\geq 0$.
\end{lemma}

\begin{proof}
The base case $k=0$ is Assumption~\ref{ass:init}.
Assuming admissibility at step $k$, the freeze case trivially preserves it, while in the update case \eqref{eq:projected_update}--\eqref{eq:filtered_update} yield $\hat{\bm{M}}_{k+1}\in\mathcal M$ (by definition of $\Proj_{\mathcal M}$) and $\bm{M}_{k+1}\in\mathcal M$ (as a convex combination of $\bm{M}_k$ and $\hat{\bm{M}}_{k+1}$, both in $\mathcal M$).
The result follows by induction.
\end{proof}

\begin{theorem}
\label{thm:uniform_diss}
Under Assumptions~\ref{ass:lipschitz}--\ref{ass:update}, let $\{\bm{M}_k\}_{k\geq 0}$ be the applied-core sequence. For any two reservoir trajectories driven by this same sequence $\{\bm{W}_k\}$, from arbitrary initial states $\bm{r}_0,\bm{r}_0'$ and under arbitrary inputs $\bm{s},\bm{s}'$, the reservoir state in \eqref{eq:reservoir} satisfies
\vspace{-3mm}
\begin{equation}
\|\bm{r}_k-\bm{r}_k'\|
    \leq a^k\|\bm{r}_0-\bm{r}_0'\| 
    \quad +\frac{\alpha\|\bm{W}^{in}\|}{1-a}\|\bm{s}-\bm{s}'\|_\infty,
\label{eq:diss_bound}
\vspace{-1mm}
\end{equation}
where $a=(1-\alpha)+\alpha\kappa<1$.
Thus the reservoir is incrementally input-to-state stable, with rate $a$ and gain $\alpha\|\bm{W}^{in}\|/(1-a)$ that are uniform over all admissible applied-core sequences $\{\bm{M}_k\}\subset\mathcal{M}$.
\end{theorem}

\begin{proof}
By Lemma~\ref{lem:admissible_invariance}, $\bm{M}_k\in\mathcal M$ for all $k\geq 0$.
By Assumption~\ref{ass:bases}, $\bm{U}^{\top}\bm{U}=\bm{V}^{\top}\bm{V}=\bm{I}$, so $\|\bm{U}\bm{M}_k\bm{V}^{\top}\|=\|\bm{M}_k\|$.
Combined with Assumption~\ref{ass:base} ($\|\bm{W}_0\|\leq\kappa_0$) and Definition~\ref{def:admissible} ($\rho_M=\kappa-\kappa_0$),
$\|\bm{W}_k\|
\leq\|\bm{W}_0\|+\|\bm{U}\bm{M}_k\bm{V}^{\top}\| 
=\|\bm{W}_0\|+\|\bm{M}_k\| 
\leq\kappa_0+\rho_M=\kappa<1$
for all $k\geq 0$.
For two trajectories driven by the same $\{\bm{W}_k\}$, define
$$
\bm{z}_k=\bm{W}_k\bm{r}_k+\bm{W}^{in}\bm{s}_k,
\quad
\bm{z}_k'=\bm{W}_k\bm{r}_k'+\bm{W}^{in}\bm{s}_k'.
$$
Using \eqref{eq:reservoir}, Assumption~\ref{ass:lipschitz} ($\sigma$ is Lipschitz with $L_\sigma\leq 1$, so $\|\sigma(\bm{z}_k)-\sigma(\bm{z}_k')\|\leq\|\bm{z}_k-\bm{z}_k'\|$), and the triangle inequality,
\begin{align*}
\|\bm{r}_{k+1}-\bm{r}_{k+1}'\|
&\leq(1-\alpha)\|\bm{r}_k-\bm{r}_k'\| \quad +\alpha\|\bm{W}_k(\bm{r}_k-\bm{r}_k')\| \\
&\quad+\alpha\|\bm{W}^{in}(\bm{s}_k-\bm{s}_k')\|\\
&\leq\bigl((1-\alpha)+\alpha\kappa\bigr)\|\bm{r}_k-\bm{r}_k'\| \\
&\quad+\alpha\|\bm{W}^{in}\|\|\bm{s}_k-\bm{s}_k'\|.
\end{align*}
Setting $a=(1-\alpha)+\alpha\kappa<1$ and applying recursively yields
\vspace{-2mm}
$$
\|\bm{r}_k-\bm{r}_k'\|
\leq a^k\|\bm{r}_0-\bm{r}_0'\|
+\alpha\|\bm{W}^{in}\|\sum_{j=0}^{k-1}a^{k-1-j}\|\bm{s}_j-\bm{s}_j'\|.
\vspace{-2mm}
$$
Since $\sum_{j=0}^{k-1}a^{k-1-j}\leq(1-a)^{-1}$, the bound \eqref{eq:diss_bound} follows.
\end{proof}

\begin{corollary}
\label{cor:rate}
Under Assumptions~\ref{ass:init}--\ref{ass:update}, the applied recurrent matrix generated by Algorithm~\ref{alg:online} satisfies
$\|\bm{W}_{k+1}-\bm{W}_k\|\leq 2\beta\rho_M$.
\end{corollary}

\begin{proof}
If the recurrent update is frozen, $\bm{M}_{k+1}=\bm{M}_k$ and $\bm{W}_{k+1}=\bm{W}_k$.
Otherwise, \eqref{eq:filtered_update} gives $\bm{M}_{k+1}-\bm{M}_k=\beta(\hat{\bm{M}}_{k+1}-\bm{M}_k)$.
By Lemma~\ref{lem:admissible_invariance}, $\|\hat{\bm{M}}_{k+1}\|\leq\rho_M$ and $\|\bm{M}_k\|\leq\rho_M$, so $\|\bm{M}_{k+1}-\bm{M}_k\|\leq 2\beta\rho_M$.
Finally, by Assumption~\ref{ass:bases},
$\|\bm{W}_{k+1}-\bm{W}_k\|
=\|\bm{U}(\bm{M}_{k+1}-\bm{M}_k)\bm{V}^{\top}\| 
=\|\bm{M}_{k+1}-\bm{M}_k\|
\leq 2\beta\rho_M$.
\end{proof}

\begin{corollary}
\label{cor:prediction_sensitivity}
Under the hypotheses of Theorem~\ref{thm:uniform_diss}, consider two trajectories driven by the same $\{\bm{W}_k\}$ and sharing the readout sequence $\{\bm{W}^{out}_k\}$, with $\|\bm{W}^{out}_k\|\leq\bar R$ for all $k\geq 0$, which the projected readout update \eqref{eq:readout_update} enforces by construction. Then their predictions $\hat{\bm{y}}_{k+1}=\bm{W}^{out}_k\bm{r}_{k+1}$ satisfy
\begin{equation*}
    \|\hat{\bm{y}}_{k+1}-\hat{\bm{y}}_{k+1}'\|
    \leq \bar R\!\left(a^{k+1}\|\bm{r}_0-\bm{r}_0'\|+\frac{\alpha\|\bm{W}^{in}\|}{1-a}\|\bm{s}-\bm{s}'\|_\infty\right).
\end{equation*}
\end{corollary}

\begin{proof}
Since $\hat{\bm{y}}_{k+1}=\bm{W}^{out}_k\bm{r}_{k+1}$, $\|\hat{\bm{y}}_{k+1}-\hat{\bm{y}}_{k+1}'\|\leq\bar R\|\bm{r}_{k+1}-\bm{r}_{k+1}'\|$, and the result follows by applying \eqref{eq:diss_bound} at index $k+1$.
\end{proof}

\section{Experiments}

LoRA-RC is evaluated against two baselines on the Lorenz-63 system
$\dot{x} = 10(y-x), 
\dot{y} = x(\rho-z)-y,
\dot{z} = xy-\tfrac{8}{3}z,$
integrated with Euler step $\Delta t=0.01$ and independent additive Gaussian noise $\sigma_w=0.1$ on each coordinate.
The parameter $\rho$ changes abruptly from $28$ to $40$ at step $k=800$, an unannounced regime shift such as payload change, plant-parameter drift, or sensor bias in CPS.
The predictor maps $\bm{s}_k=[x_k,y_k,z_k]^\top$ to $\bm{y}_{k+1}\in\mathbb{R}^3$, with the reservoir pretrained on $T_{\rm train}=700$ steps under $\rho=28$.
Three methods were compared: (i)~\textit{fixed RC}, reservoir and readout frozen after pretraining; (ii)~\textit{adaptive RC}, online readout update~\eqref{eq:readout_update} with the reservoir frozen ($\bm{M}_k\equiv 0$); and (iii)~\textit{LoRA-RC} (proposed), jointly updating the readout and the low-rank core.
All methods shared the base reservoir: $n=200$ neurons, spectral-norm bounds $\kappa_0=0.60$ and $\kappa=0.85$, leaking rate $\alpha=0.30$, and readout learning rate $\eta_R=2\times10^{-3}$.
LoRA-RC used rank $r=5$ ($2nr=2000$ offline basis parameters, $r^2=25$ online), core learning rate $\eta_M=0.04$, leakage $\lambda_M=60$, filter coefficient $\beta=0.05$, and projection radius $\rho_M=0.25$.
The bases $\bm{U},\bm{V}$ were computed by the data-informed procedure of Algorithm~\ref{alg:offline} from two drift regimes, $\rho\in\{33,40\}$.
Rebuilding them from $\rho\in\{33,36\}$, leaving the deployment regime unseen, gives post-drift RMSE $0.620\pm0.138$, cf.\ Table~\ref{tab:results}(b). The experiment does not demonstrate generalization to wider drift regimes.

\vspace{-3mm}
\begin{figure}[ht]
    \centering
    \includegraphics[width=0.95\columnwidth]{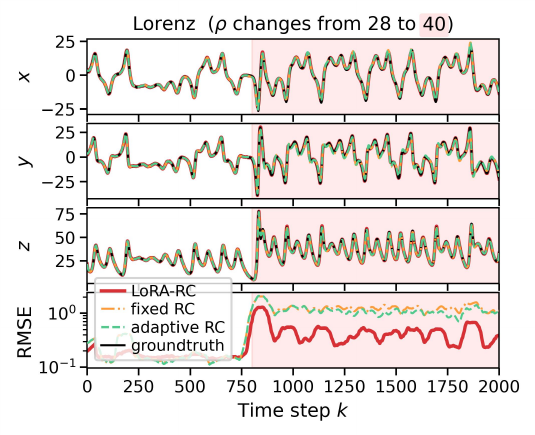}
    \vspace{-1mm}
    \caption{Lorenz state trajectories (panels 1--3) and smoothed prediction RMSE (panel 4) for fixed RC, adaptive RC, and LoRA-RC. Light red shading marks the post-drift regime ($\rho$ changes from 28 to 40).}
    \label{fig:lorenz_combined}
\end{figure}

In Fig.~\ref{fig:lorenz_combined}, the three methods track similarly before the drift; after it, fixed and adaptive RC detach from the ground truth with persistently elevated RMSE, while LoRA-RC follows the shifted attractor at a substantially lower level.
Table~\ref{tab:results}(a) quantifies both regimes. The post-drift reduction
is $56\%$ over fixed RC and $51\%$ over adaptive RC, with the lowest pre-drift error as well, so joint adaptation of the small core offers a meaningful advantage over readout adaptation alone.

With 20 random seeds, Table~\ref{tab:results}(b) reports mean $\pm$ std.
Removing the projection increases post-drift RMSE more than $40\times$, with or without the filter. Removing only the filter leaves it $5.6\times$ higher, so the filter aids accuracy rather than the certificate. Post-drift RMSE falls monotonically across $r\in\{1,2,5,10,20\}$, from $1.11\pm0.18$ to $0.41\pm0.11$, at up to $16\times$ the default's online parameters; $r=5$ is a favorable efficiency-accuracy point, not a proven optimum.
Sweeping $\rho_M$, $\beta$, $\eta_M$, and $\lambda_M$ one at a time, with the others at their Table~\ref{tab:results}(a) defaults and the same 20 seeds, gives the post-drift RMSE ranges in Table~\ref{tab:results}(c). The projection radius $\rho_M$ and core learning rate $\eta_M$ are most robust, staying in a narrow band across $7\times$ and $16\times$ ranges. The filter coefficient $\beta$ is most sensitive, increasing RMSE more than $5\times$ as $\beta \to 1$, consistent with the no-filter ablation. Large leakage $\lambda_M$ degrades performance through over-regularization.
For every parameter the default lies within, though not always at, the best-performing region.

\section{Conclusions}

LoRA-RC was developed for online reservoir adaptation in which the contraction bound on the recurrent matrix is a per-step invariant under every online update.
Building on this invariant, a uniform $\delta$ISS bound holding along every update sequence was proved with explicit contraction rate $a=(1-\alpha)+\alpha\kappa$.
Comparative simulations on the chaotic Lorenz system showed that LoRA-RC reduced post-drift RMSE by 56\% relative to fixed RC and by 51\% relative to adaptive RC, using only $r^2=25$ online recurrent parameters.
Future work will investigate Lyapunov-based conditions to enlarge the admissible set and extend the evaluation to higher-dimensional and real-world systems with closed-loop control tasks.


\bibliographystyle{IEEEtran}
\bibliography{z_refs}

@article{sontag1989smooth,
  author  = {Eduardo D. Sontag},
  title   = {Smooth Stabilization Implies Coprime Factorization},
  journal = {IEEE Transactions on Automatic Control},
  volume  = {34},
  number  = {4},
  pages   = {435--443},
  year    = {1989},
  doi     = {10.1109/9.28018}
}

@article{jiang2001discreteISS,
  author  = {Zhong-Ping Jiang and Yuan Wang},
  title   = {Input-to-State Stability for Discrete-Time Nonlinear Systems},
  journal = {Automatica},
  volume  = {37},
  number  = {6},
  pages   = {857--869},
  year    = {2001},
  doi     = {10.1016/S0005-1098(01)00028-0}
}

@article{angeli2002incremental,
  author  = {David Angeli},
  title   = {A {Lyapunov} Approach to Incremental Stability Properties},
  journal = {IEEE Transactions on Automatic Control},
  volume  = {47},
  number  = {3},
  pages   = {410--421},
  year    = {2002},
  doi     = {10.1109/9.989067}
}

@article{mayne2000mpc,
  author  = {David Q. Mayne and James B. Rawlings and Christopher V. Rao and Pierre O. M. Scokaert},
  title   = {Constrained Model Predictive Control: Stability and Optimality},
  journal = {Automatica},
  volume  = {36},
  number  = {6},
  pages   = {789--814},
  year    = {2000},
  doi     = {10.1016/S0005-1098(99)00214-9}
}

@techreport{jaeger2001echo,
  author      = {Herbert Jaeger},
  title       = {The Echo State Approach to Analysing and Training Recurrent Neural Networks},
  institution = {German National Research Center for Information Technology},
  number      = {GMD Report 148},
  year        = {2001}
}

@article{lukosevicius2009reservoir,
  author  = {Mantas Luko\v{s}evi\v{c}ius and Herbert Jaeger},
  title   = {Reservoir Computing Approaches to Recurrent Neural Network Training},
  journal = {Computer Science Review},
  volume  = {3},
  number  = {3},
  pages   = {127--149},
  year    = {2009},
  doi     = {10.1016/j.cosrev.2009.03.005}
}

@inproceedings{jaeger2002adaptive,
  author    = {Herbert Jaeger},
  title     = {Adaptive Nonlinear System Identification with Echo State Networks},
  booktitle = {Advances in Neural Information Processing Systems},
  year      = {2002}
}

@article{yang2019online,
  author  = {Cuili Yang and Junfei Qiao and Zohaib Ahmad and Kaizhe Nie and Lei Wang},
  title   = {Online Sequential Echo State Network with Sparse {RLS} Algorithm for Time Series Prediction},
  journal = {Neural Networks},
  volume  = {118},
  pages   = {32--42},
  year    = {2019},
  doi     = {10.1016/j.neunet.2019.05.006}
}

@article{gallicchio2011architectural,
  author  = {Claudio Gallicchio and Alessio Micheli},
  title   = {Architectural and {Markovian} Factors of Echo State Networks},
  journal = {Neural Networks},
  volume  = {24},
  number  = {5},
  pages   = {440--456},
  year    = {2011},
  doi     = {10.1016/j.neunet.2011.02.002}
}

@article{srinivasan2025boosting,
  author  = {Keshav Srinivasan and Dietmar Plenz and Michelle Girvan},
  title   = {Boosting Reservoir Computing with Brain-Inspired Adaptive Control of {E-I} Balance},
  journal = {Nature Communications},
  volume  = {16},
  number  = {1},
  pages   = {10212},
  year    = {2025}
}

@article{ye2025reservoir,
  author  = {Fan Ye and Arsen Abdulali and Kai-Fung Chu and Xiaoping Zhang and Fumiya Iida},
  title   = {Reservoir Controllers Design through Robot--Reservoir Timescale Alignment},
  journal = {Communications Engineering},
  volume  = {4},
  number  = {1},
  pages   = {81},
  year    = {2025}
}

@article{kim2026living,
  author  = {Seung Hyun Kim and Zhi Dou and Gaurav Upadhyay and Anay Pattanaik and Leo Maslov and Lav Varshney and John Beggs and Howard Gritton and Mattia Gazzola},
  title   = {Computing with Living Neurons: Chaos-Controlled Reservoir Computing with Knowledge Transplant},
  journal = {arXiv preprint arXiv:2604.02552},
  year    = {2026}
}

@article{yildiz2012revisiting,
  author  = {Izzet B. Yildiz and Herbert Jaeger and Stefan J. Kiebel},
  title   = {Re-visiting the Echo State Property},
  journal = {Neural Networks},
  volume  = {35},
  pages   = {1--9},
  year    = {2012},
  doi     = {10.1016/j.neunet.2012.07.005}
}

@article{manjunath2013echo,
  author  = {Gandhishankar Manjunath and Herbert Jaeger},
  title   = {Echo State Property Linked to an Input: Exploring a Fundamental Characteristic of Recurrent Neural Networks},
  journal = {Neural Computation},
  volume  = {25},
  number  = {3},
  pages   = {671--696},
  year    = {2013},
  doi     = {10.1162/NECO_a_00411}
}

@article{grigoryeva2018universal,
  author  = {Lyudmila Grigoryeva and Juan-Pablo Ortega},
  title   = {Universal Discrete-Time Reservoir Computers with Stochastic Inputs and Linear Readouts Using Non-Homogeneous State-Affine Systems},
  journal = {Journal of Machine Learning Research},
  volume  = {19},
  number  = {24},
  pages   = {1--40},
  year    = {2018}
}

@inproceedings{dong2022asymptotic,
  author    = {Jonathan Dong and Erik B{\"o}rve and Mushegh Rafayelyan and Michael Unser},
  title     = {Asymptotic Stability in Reservoir Computing},
  booktitle = {International Joint Conference on Neural Networks},
  pages     = {1--8},
  year      = {2022},
  doi       = {10.1109/IJCNN55064.2022.9892302}
}

@article{bugliari2019mpc,
  author  = {Luca {Bugliari Armenio} and Enrico Terzi and Marcello Farina and Riccardo Scattolini},
  title   = {Model Predictive Control Design for Dynamical Systems Learned by Echo State Networks},
  journal = {IEEE Control Systems Letters},
  volume  = {3},
  number  = {4},
  pages   = {1044--1049},
  year    = {2019},
  doi     = {10.1109/LCSYS.2019.2920720}
}

@article{williams2026reservoir,
  author  = {Jan P. Williams and J. Nathan Kutz and Krithika Manohar},
  title   = {Reservoir Computing for System Identification and Model Predictive Control},
  journal = {Neural Networks},
  volume  = {202},
  pages   = {109031},
  year    = {2026}
}

@article{inoue2025reservoir,
  author  = {Daisuke Inoue and Tadayoshi Matsumori and Gouhei Tanaka and Yuji Ito},
  title   = {Reservoir Predictive Path Integral Control for Unknown Nonlinear Dynamics},
  journal = {arXiv preprint arXiv:2509.03839},
  year    = {2025}
}

@article{chen2022calibrated,
  author  = {Yeyuge Chen and Yu Qian and Xiaohua Cui},
  title   = {Time Series Reconstructing Using Calibrated Reservoir Computing},
  journal = {Scientific Reports},
  volume  = {12},
  number  = {1},
  pages   = {16318},
  year    = {2022}
}

@inproceedings{neglia2026rescp,
  author    = {Roberto Neglia and Andrea Cini and Michael M. Bronstein and Filippo Maria Bianchi},
  title     = {{RESCP}: Reservoir Conformal Prediction for Time Series Forecasting},
  booktitle = {International Conference on Learning Representations},
  year      = {2026}
}

@article{lu2017reservoir,
  author  = {Zhixin Lu and Jaideep Pathak and Brian Hunt and Michelle Girvan and Roger Brockett and Edward Ott},
  title   = {Reservoir Observers: Model-Free Inference of Unmeasured Variables in Chaotic Systems},
  journal = {Chaos: An Interdisciplinary Journal of Nonlinear Science},
  volume  = {27},
  number  = {4},
  pages   = {041102},
  year    = {2017}
}

@article{tavakoli2026timedelay,
  author  = {S. Kamyar Tavakoli and J{\'e}r{\'e}mie Lefebvre and Andr{\'e} Longtin},
  title   = {Time-Delay Reservoir for Signal Demixing Using {Kalman} Weight Updates in Fixed Point and Limit Cycle Regimes},
  journal = {Scientific Reports},
  year    = {2026}
}

@article{bonassi2021gru,
  author  = {Fabio Bonassi and Marcello Farina and Riccardo Scattolini},
  title   = {On the Stability Properties of Gated Recurrent Units Neural Networks},
  journal = {Systems \& Control Letters},
  volume  = {157},
  pages   = {105049},
  year    = {2021},
  doi     = {10.1016/j.sysconle.2021.105049}
}

@article{damico2024incremental,
  author  = {William {D'Amico} and Alessio {La Bella} and Marcello Farina},
  title   = {An Incremental Input-to-State Stability Condition for a Class of Recurrent Neural Networks},
  journal = {IEEE Transactions on Automatic Control},
  volume  = {69},
  number  = {4},
  pages   = {2221--2236},
  year    = {2024},
  doi     = {10.1109/TAC.2023.3327937}
}

@article{bonassi2024nmpc,
  author  = {Fabio Bonassi and Alessio {La Bella} and Marcello Farina and Riccardo Scattolini},
  title   = {Nonlinear {MPC} Design for Incrementally {ISS} Systems with Application to {GRU} Networks},
  journal = {Automatica},
  volume  = {159},
  pages   = {111381},
  year    = {2024},
  doi     = {10.1016/j.automatica.2023.111381}
}

@article{labella2025regional,
  author  = {Alessio {La Bella} and Marcello Farina and William {D'Amico} and Luca Zaccarian},
  title   = {Regional Stability Conditions for Recurrent Neural Network-Based Control Systems},
  journal = {Automatica},
  volume  = {174},
  pages   = {112127},
  year    = {2025},
  doi     = {10.1016/j.automatica.2025.112127}
}

@inproceedings{hu2022lora,
  author    = {Edward J. Hu and Yelong Shen and Phillip Wallis and Zeyuan Allen-Zhu and Yuanzhi Li and Shean Wang and Lu Wang and Weizhu Chen},
  title     = {{LoRA}: Low-Rank Adaptation of Large Language Models},
  booktitle = {International Conference on Learning Representations},
  year      = {2022}
}

@inproceedings{zhang2023adalora,
  author    = {Qingru Zhang and Minshuo Chen and Alexander Bukharin and Nikos Karampatziakis and Pengcheng He and Yu Cheng and Weizhu Chen and Tuo Zhao},
  title     = {{AdaLoRA}: Adaptive Budget Allocation for Parameter-Efficient Fine-Tuning},
  booktitle = {International Conference on Learning Representations},
  year      = {2023}
}

@inproceedings{zou2025flylora,
  author    = {Heming Zou and Yunliang Zang and Wutong Xu and Yao Zhu and Xiangyang Ji},
  title     = {{FlyLoRA}: Boosting Task Decoupling and Parameter Efficiency via Implicit Rank-Wise Mixture-of-Experts},
  booktitle = {Advances in Neural Information Processing Systems},
  year      = {2025}
}

@inproceedings{luo2026keeplora,
  author    = {Mao-Lin Luo and Zi-Hao Zhou and Yi-Lin Zhang and Yuanyu Wan and Tong Wei and Min-Ling Zhang},
  title     = {{KeepLoRA}: Continual Learning with Residual Gradient Adaptation},
  booktitle = {International Conference on Learning Representations},
  year      = {2026}
}

\end{document}